\documentclass[letterpaper]{article}
\usepackage{flairs}
\usepackage{times}
\usepackage{helvet}
\usepackage{courier}

\usepackage{booktabs}
\usepackage{enumitem}
\usepackage{amsfonts}
\usepackage{amsthm}
\usepackage{amsmath}
\usepackage{amssymb}
\usepackage{color}
\usepackage{graphicx}
\usepackage{subfig}
\usepackage[]{hyperref}
\newtheorem{prop}{Proposition}

\begin{document}
%
\title{A Closer Look at Invalid Action Masking in Policy Gradient Algorithms}

\author{Shengyi Huang \textnormal{and} Santiago Onta\~{n}\'{o}n \thanks{Currently at Google} \\
College of Computing \& Informatics, 
Drexel University\\
Philadelphia, PA 19104 \\
\texttt{\{sh3397,so367\}@drexel.edu}
}

\maketitle
\begin{abstract}
\begin{quote}
In recent years, Deep Reinforcement Learning (DRL) algorithms have achieved state-of-the-art performance in many challenging strategy games. Because these games have complicated rules, an action sampled from the full discrete action distribution predicted by the learned policy is likely to be invalid according to the game rules (e.g., walking into a wall). The usual approach to deal with this problem in policy gradient algorithms is to ``mask out'' invalid actions and just sample from the set of valid actions. The implications of this process, however, remain under-investigated. 
In this paper, we 1) show theoretical justification for such a practice, 2) empirically demonstrate its importance as the space of invalid actions grows, and 3) provide further insights by evaluating different action masking regimes, such as removing masking after an agent has been trained using masking.
\end{quote}
\end{abstract}

\section{Introduction}


Deep Reinforcement Learning (DRL) algorithms have yielded state-of-the-art game playing agents in challenging domains such as Real-time Strategy (RTS) games~\cite{vinyals2017starcraft,vinyals2019grandmaster} and Multiplayer Online Battle Arena (MOBA) games~\cite{Berner2019Dota2W,ye2019mastering}. Because these games have complicated rules, the valid discrete action spaces of different states usually have different sizes. That is, one state might have 5 valid actions and another state might have 7 valid actions. To formulate these games as a standard reinforcement learning problem with a singular action set, previous work combines these discrete action spaces to a \emph{full discrete action space} that contains available actions of all states~\cite{vinyals2017starcraft,Berner2019Dota2W,ye2019mastering}. Although such a full discrete action space makes it easier to apply DRL algorithms, one issue is that an action sampled from this full discrete action space could be invalid for some game states, and this action will have to be discarded.

To make matters worse, some games have extremely large full discrete action spaces and an action sampled will typically be invalid. As an example, the full discrete action space of Dota~2 has a size of 1,837,080~\cite{Berner2019Dota2W}, and an action sampled might be to buy an item, which can be \emph{valid} in some game states but will become \emph{invalid} when there is not enough gold. 
To avoid repeatedly sampling invalid actions in full discrete action spaces, recent work applies policy gradient algorithms in conjunction with a technique known as invalid action masking, which ``masks out'' invalid actions and then just samples from those actions that are valid~\cite{vinyals2017starcraft,Berner2019Dota2W,ye2019mastering}. To the best of our knowledge, however, the theoretical foundations of invalid action masking have not been studied and its empirical effect is under-investigated.
\begin{figure}[t]
  \centering
    \includegraphics[width=0.3\textwidth]{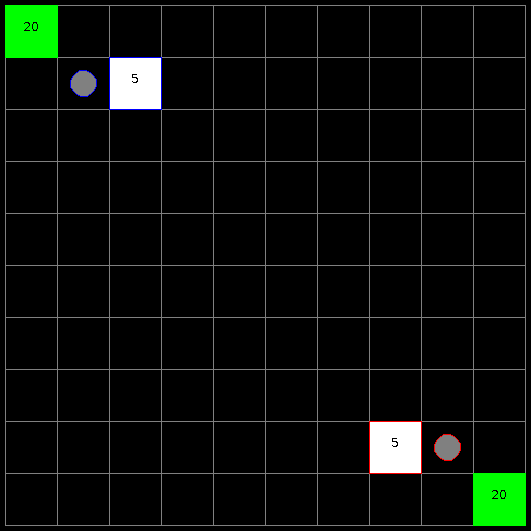}\,%
  \caption{A screenshot of $\mu$RTS. Square units are ``bases'' (light grey, that can produce workers), ``barracks'' (dark grey, that can produce military units), and ``resources mines'' (green, from where workers can extract resources to produce more units), the circular units are ``workers'' (small, dark grey) and military units (large, yellow or light blue), and on the right is the $10\times10$ map we used to train agents to harvest resources. The agents could control units at the top left, and the units in the bottom left will remain stationary.}
  \label{fig:microrts}
\end{figure}
In this paper, we take a closer look at invalid action masking in the context of games, pointing out the gradient produced by invalid action masking corresponds to a valid policy gradient. More interestingly, we show that in fact, invalid action masking can be seen as applying a \emph{state-dependent differentiable function} during the calculation of the action probability distribution, to produce a behavior policy. 
Next, we design experiments to compare the performance of {\em invalid action masking} versus {\em invalid action penalty}, which is a common approach that gives negative rewards for invalid actions so that the agent learns to maximize reward by not executing any invalid actions. We empirically show that, when the space of invalid actions grows, invalid action masking scales well and the agent solves our desired task while invalid action penalty struggles to explore even the very first reward. Then, we design experiments to answer two questions: (1) What happens if we remove the invalid action mask once the agent was trained with the mask? 
(2) What is the agent's performance when we implement the invalid action masking naively by sampling the action from the masked action probability distribution but updating the policy gradient using the unmasked action probability distribution? Finally, we made our source code available at GitHub for the purpose of reproducibility\footnote{\footnotesize \url{https://github.com/vwxyzjn/invalid-action-masking}}.

\section{Background}

In this paper, we use  policy gradient methods to train agents. 
Let us consider the Reinforcement Learning problem in a Markov Decision Process (MDP) denoted as $(S,A,P, \rho_0, r,\gamma, T)$, where $S$ is the state space, $A$ is the discrete action space, $P: S \times A \times S \rightarrow [0, 1]$ is the state transition probability, $\rho_0: S\rightarrow [0,1]$ is the initial state distribution, $r: S \times A \rightarrow \mathbb{R}$ is the reward function, $\gamma$ is the discount factor, and $T$ is the maximum episode length. A stochastic policy $\pi_{\theta}: S \times A \rightarrow [0,1]$, parameterized by a parameter vector $\theta$, assigns a probability value to an action given a state. The goal is to maximize the expected discounted return of the policy:
\begin{gather*}
        J = \mathbb{E}_{\tau}\left[\sum_{t=0}^{T-1} \gamma^{t} r_{t}\right]\\
        \text { where } \tau \text { is the trajectory } \left(s_{0}, a_{0}, r_{0},  \dots, s_{T-1}, a_{T-1}, r_{T-1}\right),\\
        \text  s_{0} \sim \rho_{0}, s_t \sim P(\cdot \vert s_{t-1}, a_{t-1}), a_t \sim \pi_{\theta}(\cdot \vert s_t), r_{t}=r\left(s_{t}, a_{t}\right)
\end{gather*}
The core idea behind policy gradient algorithms is to obtain the \textsl{policy gradient} $\nabla_{\theta}J$  of the expected discounted return with respect to the policy parameter $\theta$. Doing gradient ascent $\theta = \theta + \nabla_{\theta}J$ therefore maximizes the expected discounted reward. 
Earlier work proposes the following policy gradient estimate to the objective $J$~\cite{sutton2018reinforcement}:
\begin{align*}
   \nabla_{\theta}J = \mathbb{E}_{\tau\sim\pi_\theta}\left[ \sum_{t=0}^{T-1} \nabla_{\theta}\log\pi_{\theta}(a_t|s_t)G_t \right]\mathrm{,} \  G_{t} = \sum_{k=0}^{\infty} \gamma^{k} r_{t+k}
\end{align*}

\section{Invalid Action Masking}
Invalid action masking is a common technique implemented to avoid repeatedly generating invalid actions in large discrete action spaces~\cite{vinyals2017starcraft,Berner2019Dota2W,ye2019mastering}. To the best of our knowledge, there is no literature providing detailed descriptions of the implementation of invalid action masking. Existing work~\cite{vinyals2017starcraft,Berner2019Dota2W} seems to treat invalid action masking as an auxiliary detail, usually describing it using only a few sentences. Additionally, there is no literature providing theoretical justification to explain why it works with policy gradient algorithms. In this section, we examine how invalid action masking is implemented and prove it indeed corresponds to valid policy gradient updates~\cite{sutton2000policy}. More interestingly, we show it works by applying a \emph{state-dependent differentiable function} during the calculation of action probability distribution. 

First, let us see how a discrete action is typically generated through policy gradient algorithms. Most policy gradient algorithms employ a neural network to represent the policy, which usually outputs unnormalized scores (logits) and then converts them into an action probability distribution using a softmax operation or equivalent, which is the framework we will assume in the rest of the paper. For illustration purposes, consider an MDP with the action set $A=\{a_0, a_1, a_2, a_3\}$ and $S=\{s_0, s_1\}$, where the MDP reaches the terminal state $s_1$ immediately after an action is taken in the initial state $s_0$ and the reward is always +1. Further, consider a policy $\pi_{\theta}$ parameterized by $\theta = [l_0, l_1, l_2, l_3]=[1.0,1.0,1.0,1.0]$ that, for the sake of this example, directly produces $\theta$ as the output logits. Then in $s_0$ we have: 
\begin{align}
    \pi_{\theta}(\cdot \vert s_0)&= [\pi_{\theta}(a_0|s_0), \pi_{\theta}(a_1|s_0), \pi_{\theta}(a_2|s_0), \pi_{\theta}(a_3|s_0)]  \nonumber \\
    &=\text{softmax}([l_0, l_1, l_2, l_3])  \label{eq:softmax_logits2}\\
    &=[0.25, 0.25, 0.25, 0.25],  \nonumber\\
    &\text{where }\pi_{\theta}(a_i|s_0) = \frac{\exp(l_i)}{\sum_j \exp(l_j)} \nonumber
\end{align} 
At this point, regular policy gradient algorithms will sample an action from $\pi_{\theta}(\cdot \vert s_0)$. Suppose $a_0$ is sampled from $\pi_{\theta}(\cdot \vert s_0)$, and the policy gradient can be calculated as follows:
\begin{align*}
    \begin{aligned}
g_{\text{policy}} &=  \mathbb{E}_{\tau}\left [\nabla_{\theta} \sum_{t=0}^{T-1} \log\pi_{\theta}(a_t|s_t)G_t \right] \\
    &=\nabla_{\theta}\log\pi_{\theta}(a_0|s_0)G_0 \\
    &= [ 0.75, -0.25, -0.25, -0.25]
    \end{aligned} \\
    (\nabla_{\theta}\log \text{softmax}(\theta)_j)_i = \begin{cases}
 (1 -  \frac{\exp(l_j)}{\sum_j \exp(l_j)}) &\text{if $i=j$}\\
 \frac{-\exp(l_j)}{\sum_j \exp(l_j)} &\text{otherwise}
\end{cases}
\end{align*}
Now suppose $a_2$ is invalid for state $s_0$, and the only valid actions are $a_0, a_1, a_3$. Invalid action masking helps to avoid sampling invalid actions by ``masking out'' the logits corresponding to the invalid actions. This is usually accomplished by replacing the logits of the actions to be masked by a large negative number $M$ (e.g. $M = -1 \times 10^8$).  
Let us use $mask: \mathbb{R} \rightarrow \mathbb{R}$ to denote this masking process, and we can calculate the re-normalized probability distribution $\pi'_{\theta}(\cdot \vert s_0)$ as the following:
\begin{align}
    &\pi'_{\theta}(\cdot \vert s_0)= \text{softmax}(mask([l_0, l_1, l_2, l_3])) \label{eq:softmax_masked_logits}\\
    &=\text{softmax}([l_0, l_1, M, l_3]) \\
    &= [\pi'_{\theta}(a_0|s_0), \pi'_{\theta}(a_1|s_0), \epsilon, \pi'_{\theta}(a_3|s_0)] \label{eq:softmax_masked_logits2}\\
    &= [0.33, 0.33, 0.0000, 0.33] \nonumber
\end{align}
where $\epsilon$ is the resulting probability of the masked invalid action, which should be a small number. If $M$ is chosen to be sufficiently negative, the probability of choosing the masked invalid action $a_2$ will be virtually zero. After finishing the episode, the policy is updated according to the following gradient, which we refer to as the \emph{invalid action policy gradient}.
\begin{align}
\label{eq:invalid_action_masking_gradient}
    g_{\text{invalid action policy}} &= \mathbb{E}_{\tau}\left [\nabla_{\theta} \sum_{t=0}^{T-1} \log\pi'_{\theta}(a_t|s_t)G_t \right] \\
    &=\nabla_{\theta}\log\pi'_{\theta}(a_0|s_0)G_0  \\
    & = [ 0.67, -0.33,  0.0000, -0.33]\nonumber
\end{align}
This example highlights that invalid action masking appears to do more than just ``renormalizing the probability distribution''; it in fact makes the gradient corresponding to the logits of the invalid action to zero.

\subsection{Masking Still Produces a Valid Policy Gradient}
The action selection process is affected by a process that seems external to $\pi_\theta$ that calculates the mask. It is therefore natural to wonder how does the policy gradient theorem~\cite{sutton2000policy} apply. 
As a matter of fact, our analysis shows that the process of invalid action masking can be considered as a state-dependent differentiable function applied for the calculation of $\pi'_{\theta}$, and therefore $g_{\text{invalid action policy}}$ can be considered as a policy gradient update for $\pi'_{\theta}$.



\begin{prop}
\label{prop:policy_gradient}
$g_{\text{invalid action policy}}$ is the policy gradient of policy $\pi'_{\theta}$.
\end{prop}
\begin{proof} Let $s\in S$ to be arbitrary and consider the process of invalid action masking as a differentiable function $mask$ to be applied to the logits $l(s)$ outputted by policy $\pi_\theta$ given state $s$. Then we have:
\begin{align*}
    \pi'_{\theta}(\cdot \vert s) &= \text{softmax}(mask(l(s))) \\
    mask(l(s))_i &= \begin{cases}
     l_i &\text{if $a_i$ is valid in $s$}\\
     M &\text{otherwise}
    \end{cases}
\end{align*}
Clearly, $mask$ is either an identity function or a constant function for elements in the logits. Since these two kinds of functions are differentiable, $\pi'_{\theta}$ is differentiable to its parameters $\theta$. That is, $\frac{\partial \pi'_{\theta}(a\vert s)}{\partial \theta}$ exists for all $a\in A, s\in S$, which satisfies the assumption of  policy gradient theorem~\cite{sutton2000policy}. Hence, $g_{\text{invalid action policy}}$ is the policy gradient of policy $\pi'_{\theta}$.
\end{proof}
Note that $mask$ is \emph{not} a piece-wise linear function. If we plot $mask$, it is either an identity function or a constant function, depending on the state $s$, going from $-\infty$ to $+\infty$. We therefore call $mask$ a state-dependent differentiable function. That is, given a vector $x$  and two states $s, s'$ with different number of invalid actions available in these states, $mask(s,x) \neq mask(s', x)$.

\begin{table}[t]
\begin{small}
\centering
\caption{Observation features and action components.}
\begin{tabular}{p{2.3cm}lp{3cm}} 
\toprule
Observation Features  & Planes & Description \\
\midrule
Hit Points & 5 & 0, 1, 2, 3, $\geq 4$  \\ 
Resources & 5 & 0, 1, 2, 3, $\geq 4$  \\ 
Owner &3 & player 1, -, player 2 
\\ 
Unit Types &8 & -, resource, base, barrack, worker, light, heavy, ranged \\ 
Current Action &6& -, move, harvest, return, produce, attack\\ 
\midrule
Action Components  & Range & Description \\
\midrule
Source Unit & $[0,h \times w-1]$ & the location of the unit selected to perform an action  \\ 
Action Type & $[0,5]$ & NOOP, move, harvest, return, produce, attack  \\ 
Move Parameter & $[0,3]$ & north, east, south, west \\ 
Harvest Parameter & $[0,3]$  & north, east, south, west  \\
Return Parameter & $[0,3]$ & north, east, south, west  \\
Produce Direction Parameter & $[0,3]$ & north, east, south, west  \\
Produce Type Parameter & $[0,6]$ & resource, base, barrack, worker, light, heavy, ranged \\
Attack Target Unit & $[0,h\times w-1]$  & the location of unit that  will be attacked \\
\bottomrule
\end{tabular}
\label{tab:action-components}
\end{small}
\end{table}

\begin{figure*}[ht]
\centering
{\includegraphics[width=0.97\textwidth]{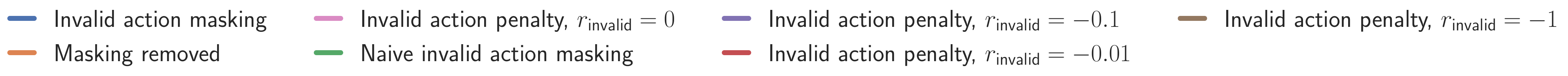}}\hfill
\subfloat[$4\times4$ Map]{\includegraphics[width=0.245\textwidth]{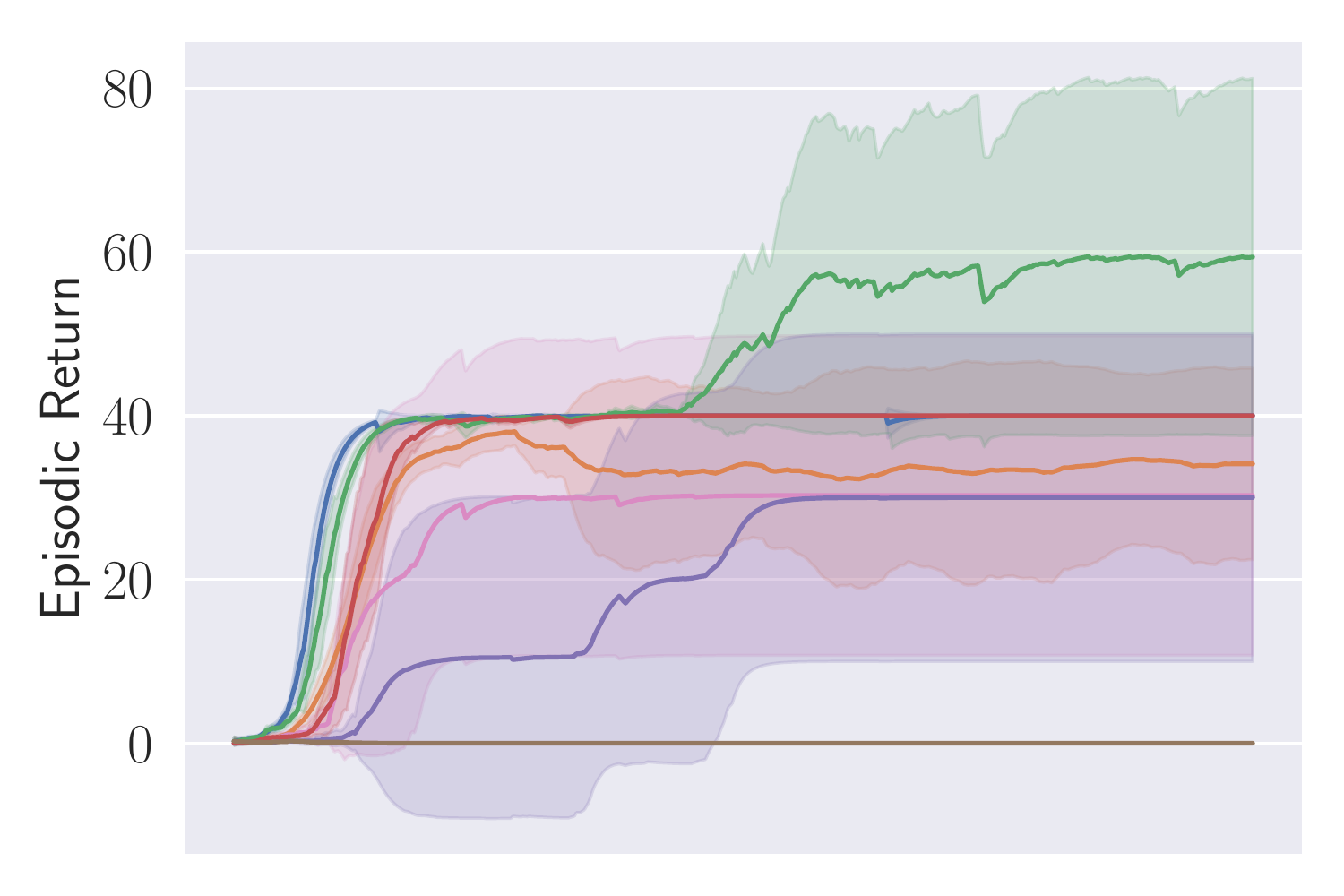}}
\subfloat[$10\times10$ Map]{\label{fig:sub:10x10}\includegraphics[width=0.245\textwidth]{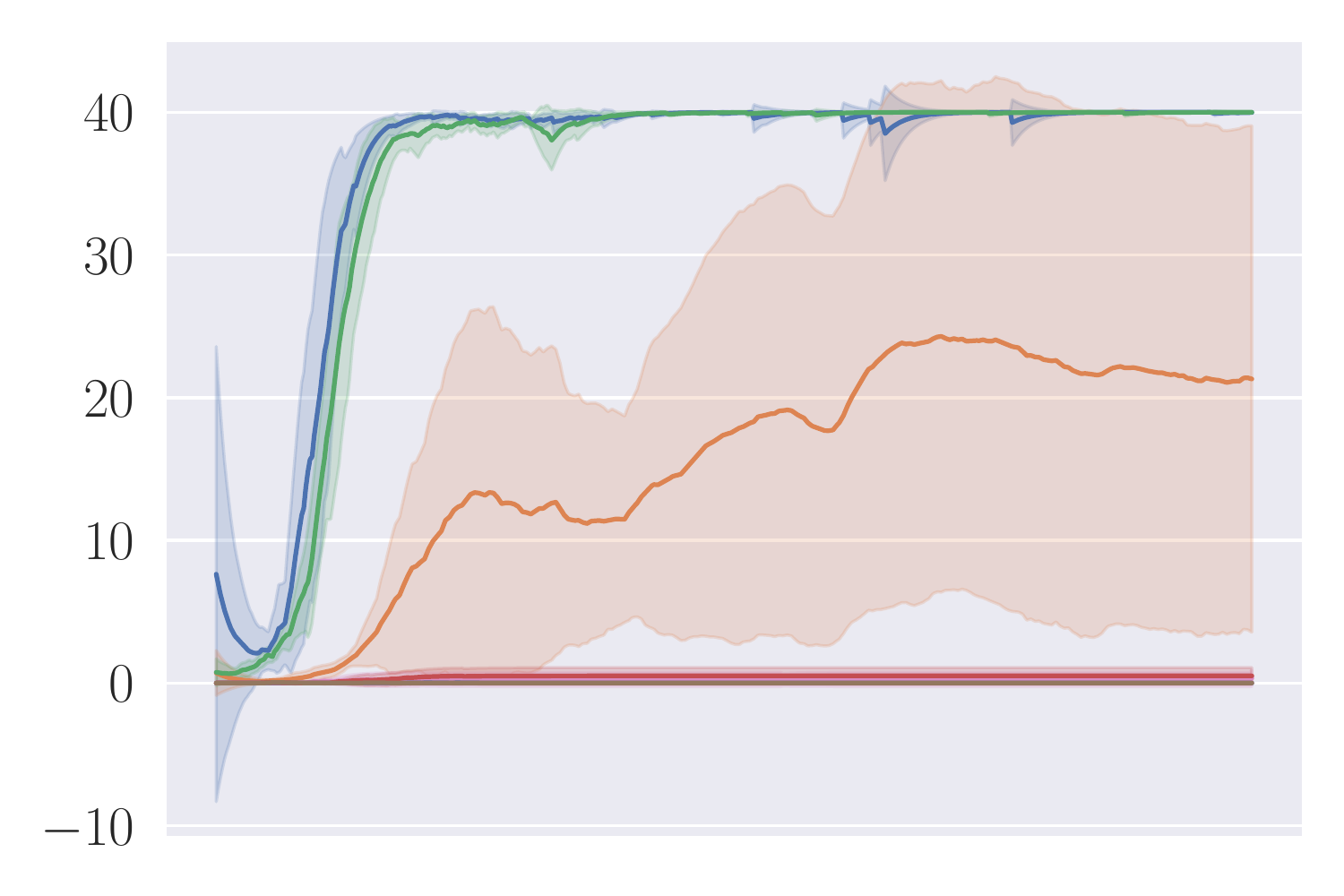}}
\subfloat[$4\times4$ Map]{\includegraphics[width=0.245\textwidth]{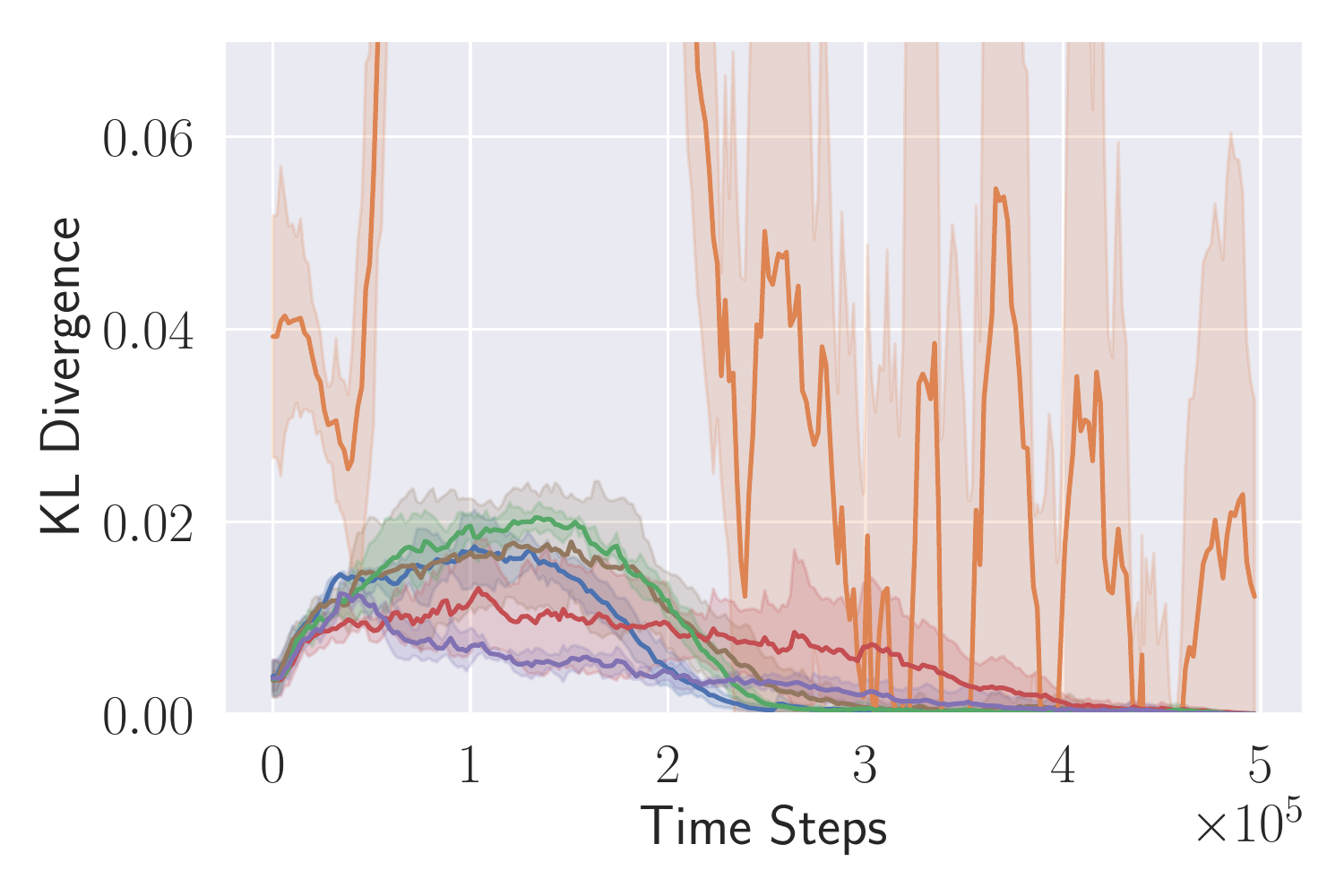}}
\subfloat[$10\times10$ Map]{\includegraphics[width=0.245\textwidth]{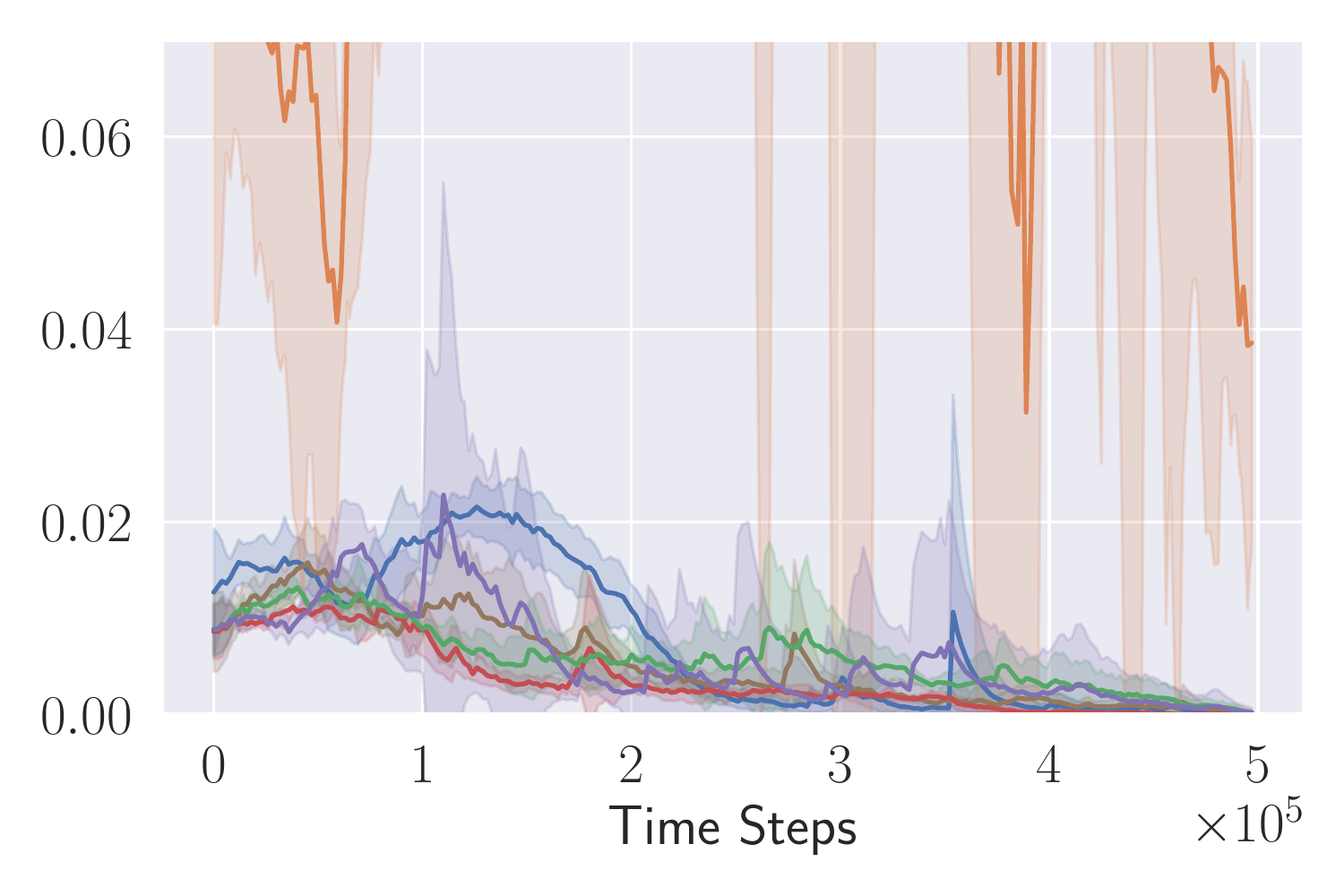}}\hfill
  \caption{The first two figures show the episodic return over the time steps, and the remaining two show the Kullback–Leibler (KL) divergence between the target and current policy of PPO over the time steps. The shaded area represents one standard deviation of the data over 4 random seeds. Curves are exponentially smoothed with a weight of 0.9 for readability.}
  \label{fig:invalid_action_masking_vs_penalty_part}
\end{figure*}

\section{Experimental Setup}
\label{sec:expsetup}
In the remainder of this paper, we provide a series of empirical results showing the practical implications of invalid action masking.


\subsection{Evaluation Environment}
\label{sec:evaluation_environment}
We use $\mu$RTS\footnote{\url{https://github.com/santiontanon/microrts}} as our testbed, which is a minimalistic RTS game maintaining the core features that make RTS games challenging from an AI point of view: simultaneous and durative actions, large branching factors, and real-time decision-making. A screenshot of the game can be found in Figure~\ref{fig:microrts}. 
It is the perfect testbed for our experiments because the action space in $\mu$RTS grows combinatorially and so does the number of invalid actions that could be generated by the DRL agent. 
We now present the technical details of the environment for our experiments.

\begin{itemize}[leftmargin=*]
    \item \textbf{Observation Space.} Given a map of size $h\times w$, the observation is a tensor of shape $(h, w, n_f)$, where $n_f$ is a number of feature planes that have binary values. The observation space used in this paper uses 27 feature planes as shown in Table~\ref{tab:action-components}, similar to previous work in $\mu$RTS~\cite{stanescu2016evaluating,Yang2018LearningME,huang2019comparing}. A feature plane can be thought of as a concatenation of multiple one-hot encoded features. As an example, if there is a worker with hit points equal to 1, not carrying any resources, the owner being Player 1, and currently not executing any actions, then the one-hot encoding features will look like the following:
    \begin{align*}
        [0,1,0,0,0],  [1,0,0,0,0],  [1,0,0], \\ [0,0,0,0,1,0,0,0],  [1,0,0,0,0,0]
    \end{align*}
    The 27 values of each feature plane for the position in the map of such worker will thus be the concatenation of the arrays above.
    \item \textbf{Action Space.} Given a map of size $h\times w$, the action is an 8-dimensional vector of discrete values as specified in Table~\ref{tab:action-components}. The action space is designed similar to the action space formulation by Hausknecht, et al.,~\cite{hausknecht2015deep}. The first component of the action vector represents the unit in the map to issue actions to, the second is the action type, and the rest of the components represent the different parameters different action types can take. Depending on which action type is selected, the game engine will use the corresponding parameters to execute the action. 
    \item \textbf{Rewards.} We are evaluating our agents on the simple task of harvesting resources as fast as they can for Player 1 who controls units at the top left of the map. A $+1$ reward is given when a worker harvests a resource, and another $+1$ is received once the worker returns the resource to a base.
    \item \textbf{Termination Condition.} We set the maximum game length to be 200 time steps, but the game could be terminated earlier if all the resources in the map are harvested first.
    
\end{itemize}


Notice that the space of invalid actions becomes significantly larger in larger maps. This is because the range of the first and last discrete values in the action space, corresponding to {\em Source Unit} and {\em Attack Target Unit} selection, grows linearly with the size of the map. To illustrate, in our experiments, there are usually only two units that can be selected as the {\em Source Unit} (the base and the worker).  Although it is possible to produce more units or buildings to be selected, the production behavior has no contribution to reward and therefore is generally not learned by the agent. Note the range of {\em Source Unit} is $4\times4=16$ and $24\times24=576$, in maps of size $4\times4$ and $24\times24$, respectively. Selecting a valid {\em Source Unit} at random has a probability of $2/16=0.125$ in the $4\times4$ map and $2/576=0.0034$ in the $24\times24$ map. With such action space, we can examine the scalability of invalid action masking.

\begin{table*}[tb]
    \centering
\caption{Results averaged over 4 random seeds. The symbol ``-'' means ``not applicable''. Higher is better for $r_{\text{episode}}$ and lower is better for $a_\text{null}$,  $a_\text{busy}$, $a_\text{owner}$,  $t_\text{solve}$, and  $t_\text{first}$.}
\begin{small}
\begin{tabular}{lllrrrrrll}
\toprule
Strategies & Map size & $r_{\text{invalid}}$ &    $r_{\text{episode}}$  &  $a_\text{null}$ &  $a_\text{busy}$ &  $a_\text{owner}$ &  $t_\text{solve}$ &  $t_\text{first}$ \\

\midrule
Invalid action penalty & $4\times4$ & -1.00 &                   0.00 &                           0.00 &                                0.00 &                                0.00 &                - &                 0.53\% \\
                             &       & -0.10 &                  30.00 &                           0.02 &                                0.00 &                                0.00 &                 50.94\% &                 0.52\% \\
                             &       & -0.01 &                  \textbf{40.00} &                           0.02 &                                0.00 &                                0.00 &                 14.32\% &                 0.51\% \\
                             &       &  0.00 &                  30.25 &                           2.17 &                                0.22 &                                2.70 &                 36.00\% &                 0.60\% \\
                             & $10\times10$ & -1.00 &                   0.00 &                           0.00 &                                0.00 &                                0.00 &                - &                 3.43\% \\
                             &       & -0.10 &                   0.00 &                           0.00 &                                0.00 &                                0.00 &                - &                 2.18\% \\
                             &       & -0.01 &                   0.50 &                           0.00 &                                0.00 &                                0.00 &                - &                 1.57\% \\
                             &       &  0.00 &                   0.25 &                          90.10 &                                0.00 &                              102.95 &                - &                 3.41\% \\
                             & $16\times16$ & -1.00 &                   0.25 &                           0.00 &                                0.00 &                                0.00 &                - &                 0.44\% \\
                             &       & -0.10 &                   0.75 &                           0.00 &                                0.00 &                                0.00 &                - &                 0.44\% \\
                             &       & -0.01 &                   1.00 &                           0.02 &                                0.00 &                                0.00 &                - &                 0.44\% \\
                             &       &  0.00 &                   1.00 &                         184.68 &                                0.00 &                                2.53 &                - &                 0.40\% \\
                             & $24\times24$ & -1.00 &                   0.00 &                          49.72 &                                0.00 &                                0.02 &                - &                 1.40\% \\
                             &       & -0.10 &                   0.25 &                           0.00 &                                0.00 &                                0.00 &                - &                 1.40\% \\
                             &       & -0.01 &                   0.50 &                           0.00 &                                0.00 &                                0.00 &                - &                 1.92\% \\
                             &       &  0.00 &                   0.50 &                         197.68 &                                0.00 &                                0.90 &                - &                 1.83\% \\ \midrule
Invalid action masking & 04x04 &  - &                  \textbf{40.00} &                           - &                                - &                               - &                  8.67\% &                \textbf{0.07\%} \\
                             & 10x10 &  - &                  \textbf{40.00} &                           - &                                - &                                - &                 11.13\% &                \textbf{0.05\%} \\
                             & 16x16 &  - &                  \textbf{40.00} &                           - &                                - &                                - &                 11.47\% &                 \textbf{0.08\%} \\
                             & 24x24 &  - &                  \textbf{40.00} &                           - &                                - &                                - &                 18.38\% &                 \textbf{0.07\%} \\ \midrule
Masking removed & 04x04 &  - &                  33.53 &                          63.57 &                                0.00 &                               17.57 &                 76.42\% &                 - \\
                             & 10x10 &  - &                  25.93 &                         128.76 &                                0.00 &                                7.75 &                 94.15\% &                 - \\
                             & 16x16 &  - &                  17.32 &                         165.12 &                                0.00 &                                0.52 &                - &                - \\
                             & 24x24 &  - &                  17.37 &                         150.06 &                                0.00 &                                0.94 &                - &                - \\ \midrule
Naive invalid action  & $4\times4$ &  - &                  \textbf{59.61} &                          - &                                - &                                - &                 11.74\% &                 \textbf{0.07\%} \\
masking                             & $10\times10$ &  - &                  \textbf{40.00} &                           - &                                - &                                - &                 13.97\% &                 \textbf{0.05\%} \\
                             & $16\times16$ &  - &                  \textbf{40.00} &                           - &                                - &                                - &                 30.59\% &                 \textbf{0.10\%} \\
                             & $24\times24$ &  - &                  \textbf{38.50} &                           - &                                - &                                - &                 49.14\% &                 \textbf{0.07\%} \\
\bottomrule
\end{tabular}
\end{small}
    \label{tab:all_results}
\end{table*}

\subsection{Training Algorithm}
We use Proximal Policy Optimization~\cite{schulman2017proximal} as the DRL algorithm to train our agents. 

\subsection{Strategies to Handle Invalid Actions}
To examine the empirical importance of invalid action masking, we compare the following four strategies to handle invalid actions.
\begin{enumerate}
    \item \textbf{Invalid action penalty.} Every time the agent issues an invalid action, the game environment adds a non-positive reward $r_{\text{invalid}} \leq 0$ to the reward produced by the current time step. This technique is standard in previous work~\cite{dietterich2000hierarchical}. We experiment with $r_{\text{invalid}} \in \{0, -0.01, -0.1, -1\}$, respectively, to study the effect of the different scales on the negative reward.
    \item \textbf{Invalid action masking.} At each time step $t$, the agent receives a mask on the {\em Source Unit} and {\em Attack Target Unit} features such that only valid units can be selected and targeted. Note that in our experiments, invalid actions still could be sampled because the agent could still select incorrect parameters for the current action type. We didn't provide a feature-complete invalid action mask for simplicity, as the mask on {\em Source Unit} and {\em Attack Target Unit} already significantly reduce the action space.
    \item \textbf{Naive invalid action masking.} At each time step $t$, the agent receives the same mask on the {\em Source Unit} and {\em Attack Target Unit} as described  for invalid action masking. The action shall still be sampled according to the re-normalized probability calculated in Equation~\ref{eq:softmax_masked_logits2}, which ensures no invalid actions could be sampled, but  the gradient is updated according to the probability calculated in Equation~\ref{eq:softmax_logits2}. We call this implementation \emph{naive invalid action masking} because its gradient does not replace the gradient of the logits corresponds to invalid actions with zero.
    
    \item \textbf{Masking removed.} At each time step $t$, the agent receives the same mask on the {\em Source Unit} and {\em Attack Target Unit} as described for invalid action masking, and trains in the same way as the agent trained under invalid action masking. However, we then evaluate the agent without providing the mask. In other words, in this scenario, we evaluate what happens if we train with a mask, but then perform without it.
\end{enumerate}

We evaluate the agent's performance in maps of sizes $4\times4$, $10\times10$, $16\times16$, and $24\times24$. All maps have one base and one worker for each player, and each worker is located near the resources. 

\subsection{Evaluation Metrics}
We used the following metrics to measure the performance of the agents in our experiments: $r_{\text{episode}}$ is the average episodic return  over the last 10 episodes. 
$a_\text{null}$ is the average number of actions that select a {\em Source Unit} that is not valid  over the last 10 episodes. $a_\text{busy}$is the average number of actions that select a {\em Source Unit} that is already busy executing other actions over the last 10 episodes. $a_\text{owner}$ is the average number of actions that select a {\em Source Unit} that does not belong to Player 1 over the last 10 episodes. $t_\text{solve}$ is the percentage of total training time steps that it takes for the agents' moving average episodic return  of the last 10 episodes to exceed 40. $t_\text{first}$ is the percentage of the total training time step that it takes for the agent to receive the first positive reward. 

\subsection{Evaluation Results}

We report the results in Figure~\ref{fig:invalid_action_masking_vs_penalty_part} and in Table~\ref{tab:all_results}. Here we present a list of important observations:

\textbf{Invalid action masking scales well. } Invalid action masking is shown to scale well as the number of invalid actions increases; $t_\text{solve}$ is roughly 12\% and very similar across different map sizes. In addition, the $t_\text{first}$ for invalid action masking is not only the lowest across all experiments (only taking about $0.05-0.08\%$ of the total time steps), but also consistent against different map sizes. This would mean the agent was able to find the first reward very quickly regardless of the map sizes.

\textbf{Invalid action penalty does not scale.} Invalid action penalty is able to achieve good results in $4\times 4$ maps, but it does not scale to larger maps. As the space of invalid action gets larger, sometimes it struggles to even find the very first reward. E.g. in the $10\times10$ map, agents trained with invalid action penalty with $r_{\text{invalid}}=-0.01$ spent 3.43\% of the entire training time just discovering the first reward, while agents trained with invalid action masking take roughly 0.06\% of the time in all maps. In addition, the hyper-parameter $r_{\text{invalid}}$ can be difficult to tune. Although having a negative $r_{\text{invalid}}$ did encourage the agents not to execute any invalid actions (e.g. $a_\text{null}$, $a_\text{busy}$, $a_\text{owner}$ are usually very close to zero for these agents), setting $r_{\text{invalid}}=-1$ seems to have an adverse effect of discouraging exploration by the agent, therefore achieving consistently the worst performance across maps.

\textbf{KL divergence explodes for naive invalid action masking.} According to Table~\ref{tab:all_results}, the $r_{\text{episode}}$ of naive invalid action masking is the best across almost 
all maps. In the $4\times4$ map, the agent trained with naive invalid action masking even learns to travel to the other side of the map to harvest additional resources. However, naive invalid action masking has two main issues: 1) As shown in the second row of Figure~\ref{fig:invalid_action_masking_vs_penalty_full}, the average Kullback–Leibler (KL) divergence~\cite{kullback1951information} between the target and current policy of PPO for naive invalid action masking is significantly higher than that of any other experiments. Since the policy changes so drastically between policy updates, the performance of naive invalid action masking might suffer when dealing with more challenging tasks. 2) As shown in Table~\ref{tab:all_results}, the $t_\text{solve}$ of naive invalid action masking is more volatile and sensitive to the map sizes. In the $24\times 24$ map, for example, the agents trained with naive invalid action masking take 49.14\% of the entire training time to converge. In comparison, agents trained with invalid action masking exhibit a consistent $t_\text{solve} \approx 12\%$  in all maps. 

\textbf{Masking removed still behaves to some extent.} As shown in Figures~\ref{fig:sub:10x10}, masking removed is still able to perform well to a certain degree. As the map size gets larger, its performance degrades and starts to execute more invalid actions by, most prominently, selecting an invalid {\em Source Unit}. Nevertheless, its performance is significantly better than that of the agents trained with invalid action penalty even though they are evaluated without the use of invalid action masking. This shows that the agents trained with invalid action masking can, to some extent, still produce useful behavior when the invalid action masking can no longer be provided.

\section{Related Work}
There have been other approaches to deal with invalid actions. Dulac-Arnold, Evans, et al.~\cite{dulac2015deep} suggest embedding discrete action spaces into a continuous action space, using nearest-neighbor methods to locate the nearest valid actions. In the field of games with natural language, others propose to train an Action Elimination Network (AEN)~\cite{zahavy2018learn} to reduce the action set.

The purpose of avoiding executing invalid actions arguably is to boost exploration efficiency. Some less related work achieves this purpose by reducing the full discrete action space to a simpler action space. Kanervisto, et al.~\cite{kanervisto2020action} describes this kind of work as ``action space shaping'', which typically involves 1) action removals (e.g. Minecraft RL environment removes non-useful actions such as ``sneak''~\cite{johnson2016malmo}), and 2) discretization of continuous action space (e.g. the famous CartPole-v0 environment discretize the continuous forces to be applied to the cart~\cite{brockman2016openai}). Although a well-shaped action space could help the agent efficiently explore and learn a useful policy, action space shaping is shown to be potentially difficult to tune and sometimes detrimental in helping the agent solve the desired tasks~\cite{dulac2015deep}.

Lastly, Kanervisto, et al.~\cite{kanervisto2020action} and Ye, et al.~\cite{ye2019mastering} provide ablation studies to show invalid action masking could be important to the performance of agents, but they do not study the empirical effect of invalid action masking as the space of invalid action grows, which is addressed in this paper.

\section{Conclusions}
In this paper, we examined the technique of invalid action masking, which is a technique commonly implemented in policy gradient algorithms to avoid executing invalid actions. Our work shows that: 1) the gradient produced by invalid action masking is a valid policy gradient, 2) it works by applying a \emph{state-dependent differentiable function} during the calculation of action probability distribution, 3) invalid action masking empirically scales well as the space of invalid action gets larger; in comparison, the common technique of giving a negative reward when an invalid action is issued fails to scale, sometimes struggling to find even the first reward in our environment, 4) the agent trained with invalid action masking was still able to produce useful behaviors with masking removed. 

Given the clear effectiveness of invalid action masking demonstrated in this paper, we believe the community would benefit from wider adoption of this technique in practice.  Invalid action masking empowers the agents to learn more efficiently, and we ultimately hope that it will accelerate research in applying DRL to games with large and complex discrete action spaces.






\fontsize{9.5pt}{10.5pt}
\selectfont
\bibliographystyle{flairs}
\bibliography{references}

\clearpage
\onecolumn

\appendix
\section*{Appendices}
\addcontentsline{toc}{section}{Appendices}
\renewcommand{\thesubsection}{\Alph{subsection}}

\subsection{Details on the Training Algorithm Proximal Policy Optimization}
\label{sec:details_on_ppo}
The DRL algorithm that we use to train the agent is Proximal Policy Optimization (PPO)~\cite{schulman2017proximal}, one of the state of the art algorithms available. There are two important  details regarding our PPO implementation that warrants explanation, as those details are not elaborated in the original paper. The first detail concerns how to generate an action in the \texttt{MultiDiscrete}  action space as defined in the OpenAI Gym environment~\cite{brockman2016openai} of gym-microrts~\cite{huang2019comparing}, while the second detail is about the various code-level optimizations utilized to augment performance. As pointed out by Engstrom, Ilyas, et al.~\cite{engstrom2019implementation}, such code-level optimizations could be critical to the performance of PPO.

\subsubsection{Multi Discrete Action Generation}
To perform an action $a_t$ in $\mu$RTS, according to Table~\ref{tab:action-components}, we have to select a Source Unit, Action Type, and its corresponding action parameters. So in total, there are $hw\times6\times4\times4\times4\times4\times6\times hw = 9216(hw)^2 $  number of possible discrete actions (including invalid ones), which grows exponentially as we increase the map size. If we apply the PPO directly to this discrete action space, it would be computationally expensive to generate the distribution for $9216(hw)^2 $ possible actions. To simplify this combinatorial action space, \texttt{openai/baselines}~\cite{baselines} library proposes an idea to consider this discrete action to be composed from some smaller \emph{independent} discrete actions. Namely, $a_t$ is composed of smaller actions 
\begin{align*}
    &a_{t}^{\text{Source Unit}},a_{t}^{\text{Action Type}},a_{t}^{\text{Move Parameter}},a_{t}^{\text{Harvest Parameter}}, \\
    &a_{t}^{\text{Return Parameter}},a_{t}^{\text{Produce Direction Parameter}}, a_{t}^{\text{Produce Type Parameter}},a_{t}^{\text{Attack Target Unit}}
\end{align*}
And the policy gradient is updated in the following way (without considering the PPO's clipping for simplicity)
\begin{align*}
     &\begin{aligned}
         \sum_{t=0}^{T-1}\nabla_{\theta}\log\pi_{\theta}(a_t|s_t)G_t  &= \sum_{t=0}^{T-1}\nabla_{\theta}  \left( \sum_{d\in D} \log\pi_{\theta}(a^{d}_{t}|s_t) \right)G_t\\
         &= \sum_{t=0}^{T-1}\nabla_{\theta}  \log \left( \prod_{d\in D} \pi_{\theta}(a^{d}_{t}|s_t) \right)G_t
     \end{aligned} \\
     &D = \{\text{Source Unit},\text{Action Type},\text{Move Parameter},\text{Harvest Parameter},\text{Return Parameter},\\
     &\text{Produce Direction Parameter},\text{Produce Type Parameter},\text{Attack Target Unit},\}
\end{align*}

Implementation wise, for each Action Component of range $[0, x-1]$, the logits of the corresponding shape $x$ is generated, which we call Action Component logits, and each $a^{d}_{t}$ is sampled from this Action Component logits. Because of this idea, the algorithm now only has to generate $hw+6+4+4+4+4+6+hw = 2hw + 36$ number of logits, which is significantly less than $9216(hw)^2$. To the best of our knowledge, this approach of handling large multi discrete action space is only mentioned by Kanervisto et, al~\cite{kanervisto2020action}.

\subsubsection{Code-level Optimizations}
Here is a list of code-level optimizations utilized in this experiments. For each of these optimizations, we include a footnote directing the readers to the files in the  \emph{openai/baselines}~\cite{baselines} that implements these optimization.

\begin{enumerate}
    \item \textbf{Normalization of Advantages\footnote{\url{https://github.com/openai/baselines/blob/ea25b9e8b234e6ee1bca43083f8f3cf974143998/baselines/ppo2/model.py\#L139}}:} After calculating the advantages based on GAE, the advantages vector is normalized by subtracting its mean and divided by its standard deviation.
    \item \textbf{Normalization of Observation\footnote{\url{https://github.com/openai/baselines/blob/ea25b9e8b234e6ee1bca43083f8f3cf974143998/baselines/common/vec_env/vec_normalize.py\#L4}}:} The observation is pre-processed before feeding to the PPO agent. The raw observation was normalized by subtracting its running mean and divided by its variance; then the raw observation is clipped to a range, usually $[-10,10]$.
    \item \textbf{Rewards Scaling\footnote{\url{https://github.com/openai/baselines/blob/ea25b9e8b234e6ee1bca43083f8f3cf974143998/baselines/common/vec_env/vec_normalize.py\#L4}}:} Similarly, the reward is pre-processed by dividing the running variance of the discounted the returns, following by clipping it to a range, usually $[-10,10]$.
    \item \textbf{Value Function Loss Clipping\footnote{\url{https://github.com/openai/baselines/blob/ea25b9e8b234e6ee1bca43083f8f3cf974143998/baselines/ppo2/model.py\#L68-L75}}:} The PPO implementation of {\em  openai/baselines} clips the value function loss in a manner that is similar to the PPO's clipped surrogate objective:

    \[V_{loss} =\max \left[\left(V_{\theta_{t}}-V_{t a r g}\right)^{2},\left(V_{\theta_{t-1}} + \mbox{clip}\left(V_{\theta_{t}}-V_{\theta_{t-1}}, -\varepsilon, \varepsilon\right)\right)^{2}\right]\]
    where $V_{t a r g}$ is calculated by adding $V_{\theta_{t-1}}$ and the  $A$ calculated by General Advantage Estimation\cite{schulman2015high}.
    \item \textbf{Adam Learning Rate Annealing\footnote{\url{https://github.com/openai/baselines/blob/ea25b9e8b234e6ee1bca43083f8f3cf974143998/baselines/ppo2/ppo2.py\#L135}}:} The Adam \cite{kingma2014adam} optimizer's learning rate is set to decay as the number of timesteps agent trained increase.
    \item \textbf{Mini-batch updates\footnote{\url{https://github.com/openai/baselines/blob/ea25b9e8b234e6ee1bca43083f8f3cf974143998/baselines/ppo2/ppo2.py\#L160-L162}}:} The PPO implementation of the {\em  openai/baselines} also uses minibatches to compute the gradient and update the policy instead of the whole batch data such as in {\em open/spinningup}.
    The mini-batch sampling scheme, however, still makes sure that every transition is sampled only once, and that the all the transitions sampled are actually for the network update. 
    \item \textbf{Global Gradient Clipping\footnote{\url{https://github.com/openai/baselines/blob/ea25b9e8b234e6ee1bca43083f8f3cf974143998/baselines/ppo2/model.py\#L107}}:} For each update iteration in an epoch, the gradients of the policy and value network are clipped so that the ``global $\ell_{2}$ norm'' (i.e. the norm of the concatenated gradients of all parameters) does not exceed 0.5.
    \item \textbf{Orthogonal Initialization  of weights\footnote{\url{https://github.com/openai/baselines/blob/ea25b9e8b234e6ee1bca43083f8f3cf974143998/baselines/a2c/utils.py\#L58}}:} The weights and biases of fully connected layers use with orthogonal initialization scheme with different scaling. For our experiments, however, we always use the scaling of 1 for historical reasons. 
\end{enumerate}



\subsection{Additional Details on the $\mu$RTS Environment Setup}
\label{appendix:additional-detals-murts}
Each action in $\mu$RTS takes some internal game time, measured in ticks, for the action to be completed. \emph{gym-microrts} \cite{huang2019comparing} sets the time of performing harvest action, return action, and move action to be 10 game ticks. Once an action is issued to a particular unit, the unit would be considered as a ``busy'' unit and would therefore no longer be able to execute any actions until its current action is finished. To prevent the DRL algorithms from repeatedly issuing actions to ``busy'' units, \emph{gym-microrts} allows performing frame skipping of 9 frames such that from the agent's perspective, once it executes the harvest action, return action, or move action given the current observation, those actions would be finished in the next observation. Such frame skipping is used for all of our experiments.

\begin{figure*}[t]
\centering
{\includegraphics[width=0.97\textwidth]{charts_episode_reward/legend.pdf}}\hfill
\subfloat[$4\times4$ Map]{\includegraphics[width=0.245\textwidth]{charts_episode_reward/MicrortsMining4x4F9-v0.pdf}}
\subfloat[$10\times10$ Map]{\label{fig:sub:10x10}\includegraphics[width=0.245\textwidth]{charts_episode_reward/MicrortsMining10x10F9-v0.pdf}}
\subfloat[$16\times16$ Map]{\label{fig:sub:10x10}\includegraphics[width=0.245\textwidth]{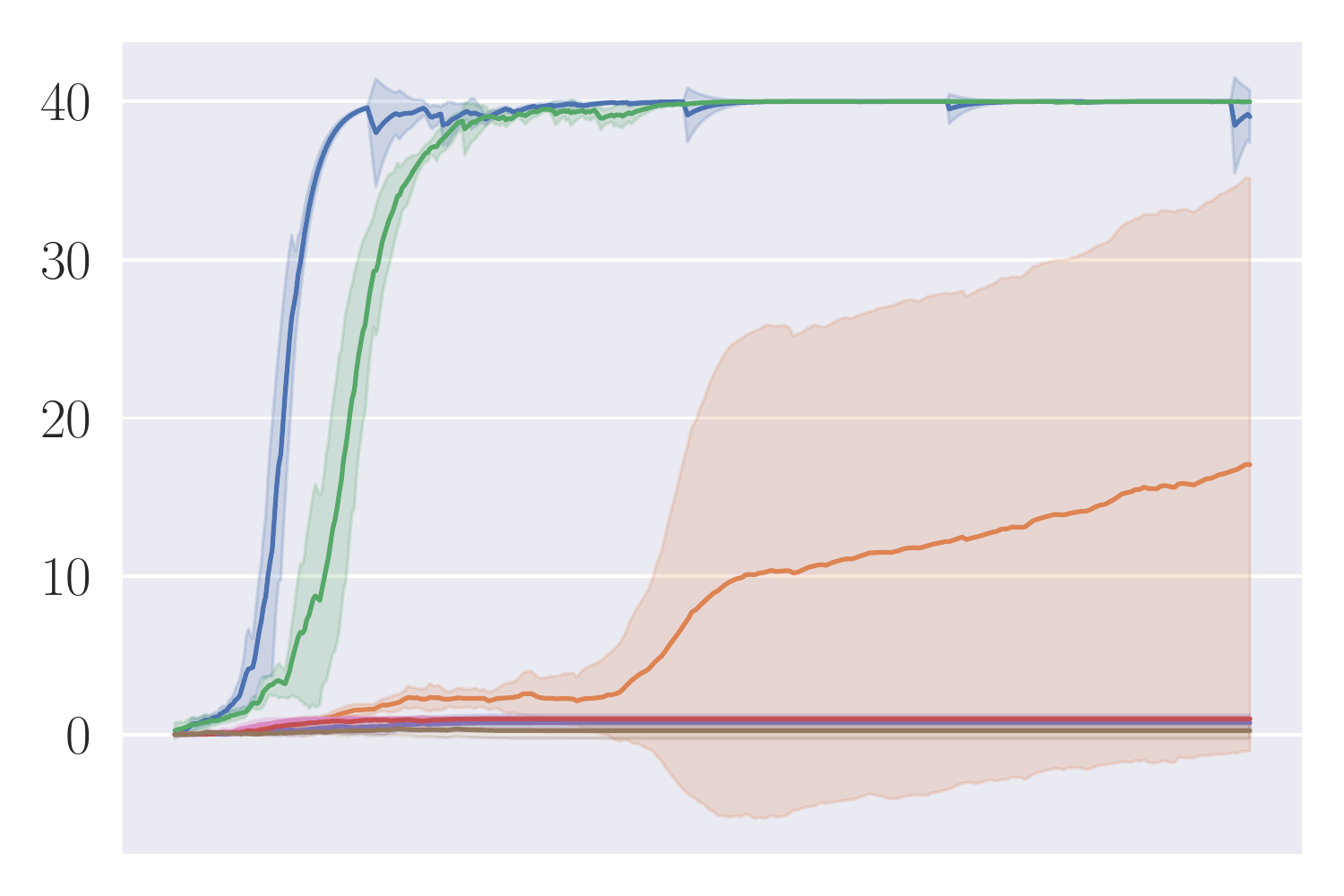}}
\subfloat[$24\times24$ Map]{\label{fig:sub:10x10}\includegraphics[width=0.245\textwidth]{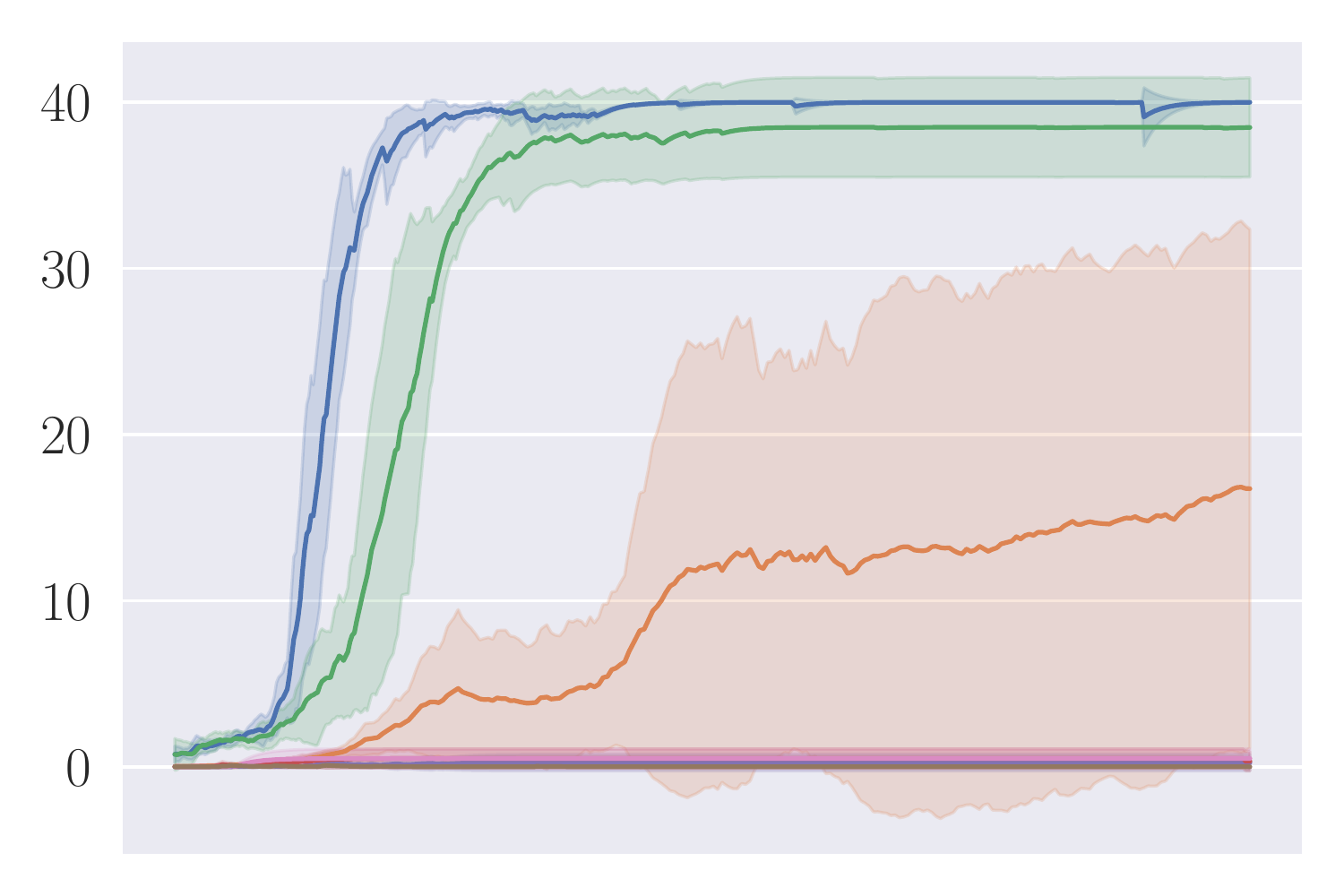}}\hfill\\
\subfloat[$4\times4$ Map]{\includegraphics[width=0.245\textwidth]{losses_approx_kl/MicrortsMining4x4F9-v0.pdf}}
\subfloat[$10\times10$ Map]{\includegraphics[width=0.245\textwidth]{losses_approx_kl/MicrortsMining10x10F9-v0.pdf}}
\subfloat[$16\times16$ Map]{\includegraphics[width=0.245\textwidth]{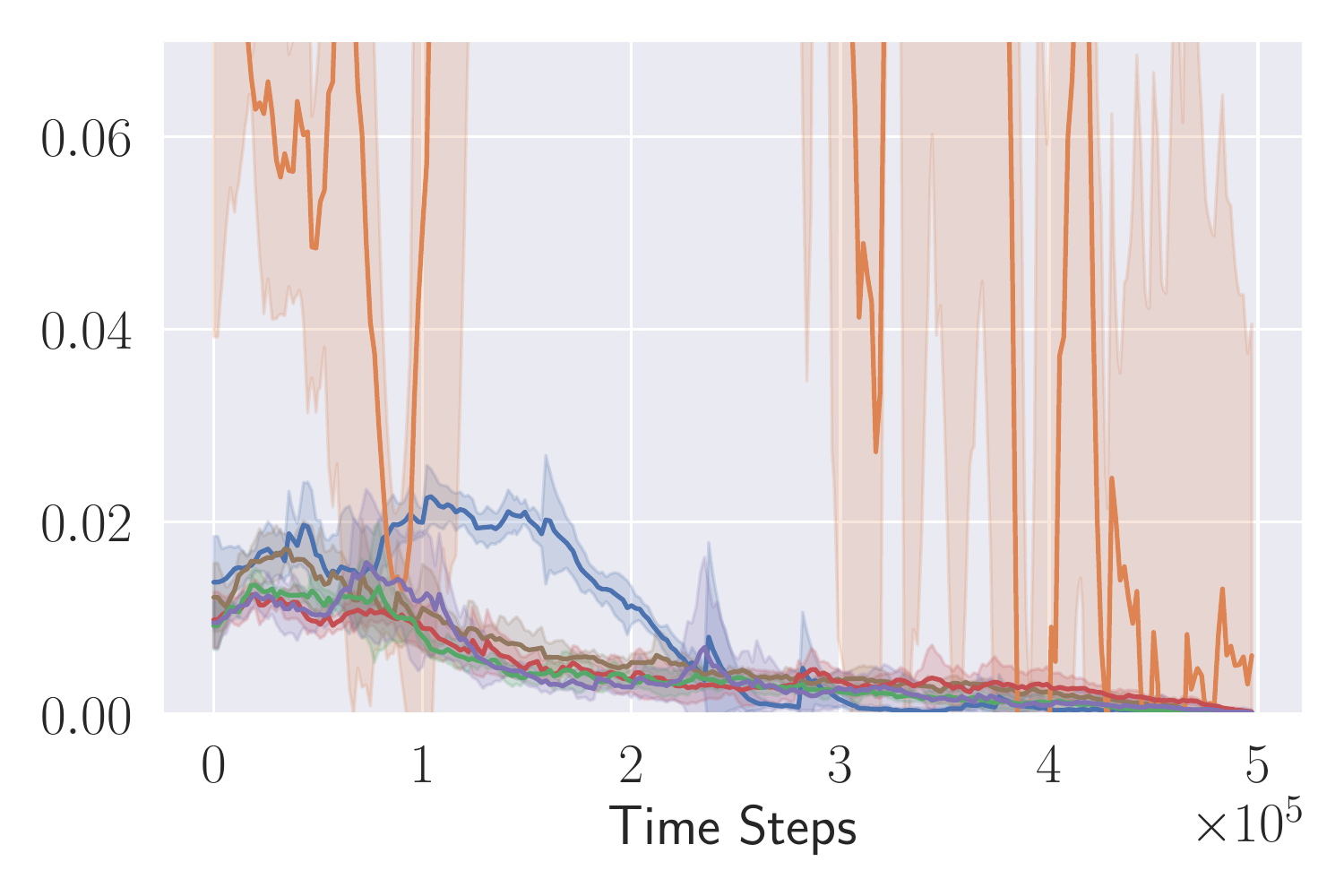}}
\subfloat[$24\times24$ Map]{\includegraphics[width=0.245\textwidth]{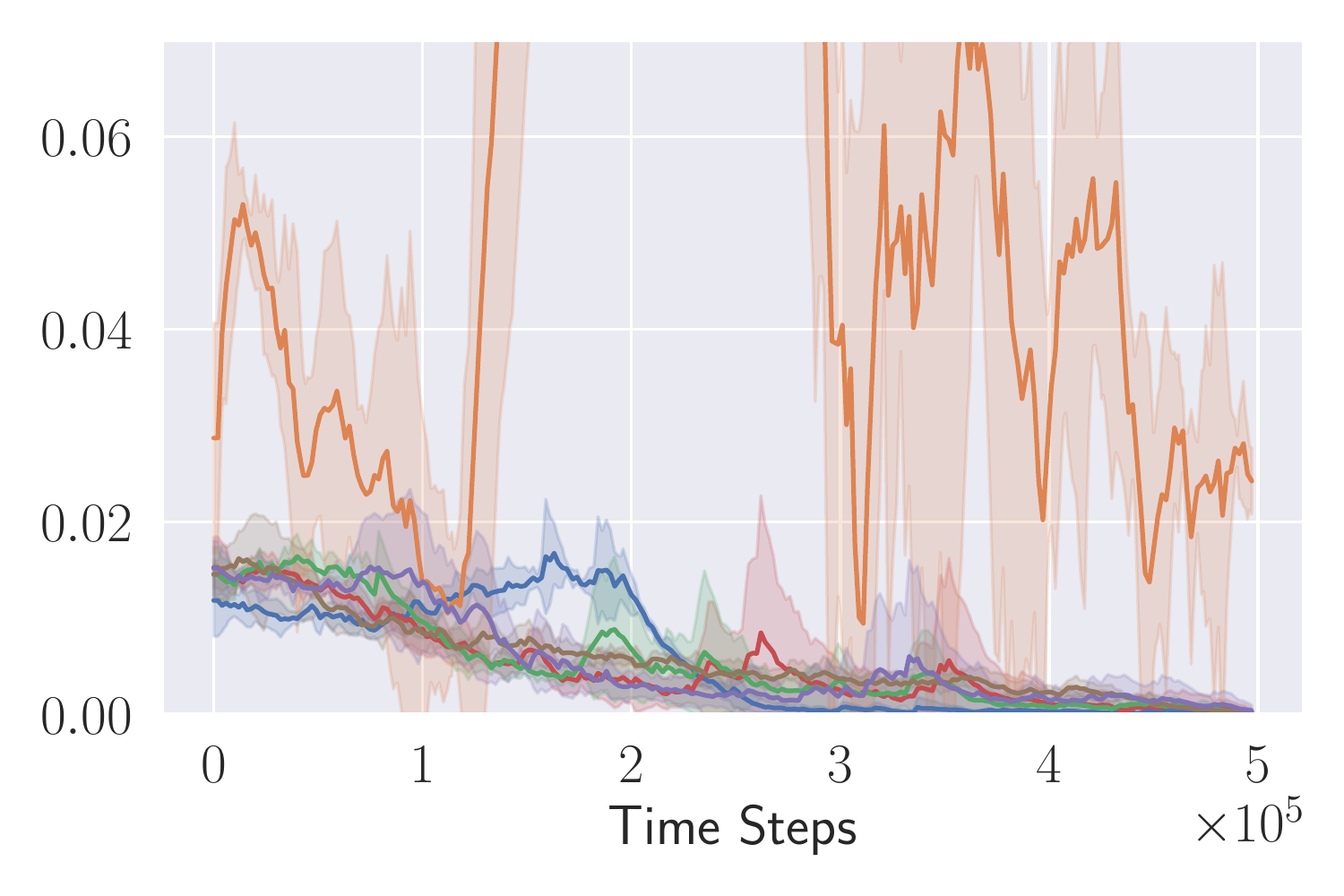}}
  \caption{The first row shows the episodic return over the time steps, and the second row shows the Kullback–Leibler (KL) divergence between the target and current policy of PPO over the time steps. The shaded area represents one standard deviation of the data over 4 random seeds. Curves are exponentially smoothed with a weight of 0.9 for readability.}
  \label{fig:invalid_action_masking_vs_penalty_full}
\end{figure*}

\subsection{Reproducibility}
It is important to for the research work to be reproducible. We now present the list of hyperparameters used in Table~\ref{tab:params} and the list of neural network architecture in  Table~\ref{tab:architecture}. In addition, we provide the source code to reproduce our experiments at GitHub\footnote{\url{https://github.com/neurips2020submission/invalid-action-masking}}.
\begin{table}[t]
\centering
\caption{The list of feature maps and their descriptions.}
\begin{tabular}{lll} 
\toprule
Features  & Planes & Description \\
\midrule
Hit Points & 5 & 0, 1, 2, 3, $\geq 4$  \\ 
Resources & 5 & 0, 1, 2, 3, $\geq 4$  \\ 
Owner &3 & player 1, -, player 2 
\\ 
Unit Types &8 & -, resource, base, barrack, \\
&&worker, light, heavy, ranged \\ 
Current Action &6& -, move, harvest, \\
&&return, produce, attack\\ 
\bottomrule
\end{tabular}
\label{tab:features}
\end{table}
\begin{table}[t]
\centering
\caption{The list of experiment parameters and their values.}
\begin{tabular}{ll} 
\toprule
Parameter Names  & Parameter Values\\
\midrule
Total Time steps & 500,000 time steps \\ 
$\gamma$ (Discount Factor) & 0.99 \\ 
$\lambda$ (for GAE) & 0.97 \\ 
$\varepsilon$ (PPO's Clipping Coefficient) & 0.2 \\ 
$\eta$ (Entropy Regularization Coefficient) & 0.01 \\ 
$\omega$ (Gradient Norm Threshold)& 0.5 \\
$K$ (Number of PPO Update Iteration Per Epoch)& 10 \\
$\alpha_{\pi}$ Policy's Learning Rate &  0.0003 \\
$\alpha_{v}$ Value Function's Learning Rate &  0.0003 \\
\bottomrule
\end{tabular}
\label{tab:params}
\end{table}  

\begin{table}[t]
\centering
\caption{Neural Network Architecture. To explain the notation, let us provide detailed description of the architecture used in $24\times24$ map as an example. The input to the neural network is a tensor of shape $(24, 24, 27)$. 
The first hidden layer convolves 16 $3\times3$ filters with stride 1 with the input tensor followed by a $2\times 2$ max pooling layer~\protect\cite{ranzato2007unsupervised} 
and applies a  rectifier nonlinearity~\protect\cite{nair2010rectified}. The second hidden layer similarly convolves 32 $2\times2$ filters with stride 1 followed by a $2\times 2$ max pooling layer and applies a  rectifier nonlinearity. The final hidden layer is a fully connected linear layer consisting of 128 rectifier units. The output layer is a fully connected linear layer with $2hw + 36=1188$ number of output.
}
\begin{tabular}{ll} 
\toprule
$4\times4$ &$10\times10$ \\
\midrule
Conv2d(27, 16, kernel\_size=2,),&   Conv2d(27, 16, kernel\_size=3,), \\
MaxPool2d(1),&                     MaxPool2d(1), \\
ReLU()&                            ReLU(), \\
Flatten()&                         Conv2d(16, 32, kernel\_size=3), \\
Linear(144, 128),&              MaxPool2d(1), \\
ReLU(),&                           ReLU() \\
Linear(128, 68)&         Flatten() \\
                                   &Linear(1152, 128), \\
                                   &ReLU(), \\
                                   &Linear(128, $236$) \\
\midrule
$16\times16$ &$24\times24$ \\
\midrule
Conv2d(27, 16, kernel\_size=3),&  Conv2d(27, 16, kernel\_size=3, stride=1), \\
MaxPool2d(1),&                   MaxPool2d(2), \\
ReLU(),&                         ReLU(), \\
Conv2d(16, 32, kernel\_size=3),&  Conv2d(16, 32, kernel\_size=2, stride=1), \\
MaxPool2d(1),&                   MaxPool2d(2), \\
ReLU()&                          ReLU() \\
Flatten()&                       Flatten() \\
Linear(4608, 128),&          Linear(800, 128), \\
ReLU(),&                         ReLU(), \\
Linear(128, 548)&       Linear(128, 1188) \\
\bottomrule
\end{tabular}

\label{tab:architecture}
\end{table}

\end{document}